%% file: hlt.tex
\documentclass{svproc}

\usepackage{url}

\usepackage{amsmath}
\usepackage{amsfonts}
\usepackage{hyperref}
\usepackage{graphicx}
\usepackage{algorithm}
\usepackage{algpseudocode}
\usepackage{relsize}

\usepackage{makecell}
\usepackage{multicol}
\usepackage{booktabs}
\usepackage{subfiles}

\usepackage[rightcaption,wide]{sidecap}

\usepackage[numbers, sort, sectionbib]{natbib}
\DeclareMathOperator*{\argmax}{arg\,max}
\DeclareMathOperator*{\argmin}{arg\,min}

\input{custom_commands.tex}

\begin{document}
\mainmatter              

\title{Hybrid Control for Learning Motor Skills}
\titlerunning{Hybrid Control for Learning Motor Skills}  

\author{Ian Abraham\inst{1} \and Alexander Broad\inst{2} \and Allison Pinosky\inst{1} \and \\ Brenna Argall\inst{1,2} \and Todd D. Murphey\inst{1}
    \thanks{ \smaller
        This material is based upon work supported by the National Science Foundation under Grants CNS 1837515.
        Any opinions, findings and conclusions or recommendations expressed in this material are those of the authors and
        do not necessarily reflect the views of the aforementioned institutions.
        For videos of results and code please visit \url{https://sites.google.com/view/hybrid-learning-theory}.
        }
}
\authorrunning{Abraham et al.} 
%
\tocauthor{Ian Abraham, Alexander Broad,
Allison Pinosky, Brenna Argall, and Todd D. Murphey}

\institute{
    Department of Mechanical Engineering \\
    \and
    Department of Electrical Engineering \\
    and Computer Science \\
    Northwestern University, Evanston, IL 60208, USA \\
    \email{i-abr@u.northwestern.edu, alexsbroad@gmail.com,
    apinosky@u.northwestern.edu, t-murphey@northwestern.edu, brenna.argall@northwestern.edu}
}

\maketitle

\vspace{-10mm}

\begin{abstract}

    We develop a hybrid control approach for robot learning based on combining learned predictive models with
    experience-based state-action policy mappings to improve the learning capabilities of robotic systems.  Predictive
    models provide an understanding of the task and the  physics (which improves sample-efficiency), while
    experience-based policy mappings are treated as ``muscle memory'' that encode favorable actions as experiences that
    override planned actions. Hybrid control tools are used to create an algorithmic approach for combining learned
    predictive models with experience-based learning. Hybrid learning is presented as a method for efficiently learning
    motor skills by systematically combining and improving the performance of predictive models and experience-based
    policies. A deterministic variation of hybrid learning is derived and extended into a stochastic implementation
    that relaxes some of the key assumptions in the original derivation. Each variation is tested on experience-based
    learning methods (where the robot interacts with the environment to gain experience) as well as imitation learning
    methods (where experience is provided through demonstrations and tested in the environment). The results show that
    our method is capable of improving the performance and sample-efficiency of learning motor skills in a variety of
    experimental domains.

\vspace{-10mm}
\end{abstract}

\vspace{-3mm}
\section{Introduction}
\vspace{-2mm}

    Model-based learning methods possess the desirable trait of being being
    ``sample-efficient''~\cite{williams2017information, chua2018deep, abraham2020modelbasedgen} in solving robot
    learning tasks. That is, in contrast to experience-based methods (i.e., model-free methods that learn a direct
    state-action map based on experience), model-based methods require significantly less data to learn a control
    response to achieve a desired task. However, a downside to model-based learning is that these models are often
    highly complex and require special structure to model the necessary intricacies~\cite{nagabandi2018neural,
    havens2019learning,
    sharma2019dynamics,abraham2020modelbasedgen,Abraham-RSS-17,abraham2019active}. Experience-based methods, on the
    other hand, approach the problem of learning skills by avoiding the need to model the environment (or dynamics) and
    instead learn a mapping (policy) that returns an action based on prior experience~\cite{schulman2017proximal,
    haarnoja2018soft}. Despite having better performance than model-based methods, experience-based approaches require a
    significant amount of data and diverse experience to work properly~\cite{chua2018deep}. How is it then, that humans
    are capable of rapidly learning tasks with a limited amount of experience? And is there a way to enable robotic
    systems to achieve similar performance?

    Some recent work tries to address these questions by exploring ``how'' models of environments are structured by
    combining probabilistic models with deterministic components~\cite{chua2018deep}. Other work has explored using
    latent-space representations to condense the complexities~\cite{havens2019learning,
    sharma2019dynamics}. Related methods use high fidelity Gaussian Processes to create models, but are limited by the
    amount of data that can be collected~\cite{deisenroth2011pilco}. Finally, some researchers try to improve
    experience-based methods by adding exploration as part of the objective~\cite{pathak2017curiosity}. However, these
    approaches often do not formally combine the usage of model-based planning with experience-based learning.

    Those that do combine model-based planning and experience-based learning tend to do so in
    stages~\cite{chebotar2017path, bansal2017mbmf, nagabandi2018neural}. First, a
    model is used to collect data for a task to jump-start what data is collected. Then, supervised learning is used to
    update a policy~\cite{levine2014GPS, chebotar2017path} or an experience-based method is used
    to continue the learning from that stage~\cite{nagabandi2018neural}. While novel, this approach does not
    algorithmically combine the two robot learning approaches in an optimal manner. Moreover, the model is often used as
    an oracle, which provides labels to a base policy. As a result, the model-based method is not improved, and the
    resulting policy is under-utilized. Our approach is to algorithmically combine model-based and experience-based learning
    by using the learned model as a gauge for how well an experience-based policy will behave, and then optimally update
    the resulting actions. Using hybrid control as the foundation for our approach, we derive a controller that
    optimally uses model-based actions when the policy is uncertain, and allows the algorithm to fall back on the
    experience-based policy when there exists high confidence actions that will result in a favorable outcome. As a
    result, our approach does not rely on improving the model (but can easily integrate high fidelity models), but
    instead optimally combines the policy generated from model-based and experience-based methods to achieve high
    performance. Our contribution can be summed up as the following:

    \begin{itemize}
        \item A hybrid control theoretic approach to robot learning
        \item Deterministic and stochastic algorithmic variations
        \item A measure for determining the agreement between learned model and policy
        \item Improved sample-efficiency and robot learning performance
        \item Diverse implementation using standard off-policy reinforcement learning~\cite{haarnoja2018soft} and behavior cloning~\cite{pomerleau1998autonomous}
    \end{itemize}

    The paper is structured as follows: Section~\ref{sec:background} provides background knowledge of the problem
    statement and its formulation; Section~\ref{sec:hybrid_learning_theory} introduces our approach and derives
    both deterministic and stochastic variations of our algorithm; Section~\ref{sec:results} provides simulated
    results and comparisons as well as experimental validation of our approach; and Section~\ref{sec:conclusion}
    concludes the work.

\vspace{-3mm}
\section{Background} \label{sec:background}
\vspace{-2mm}

    \paragraph{Markov Decision Processes:} The robot learning problem is often formulated as a Markov Decision process
    (MDP) $\MDP = \{ \sspace, \aspace, r, p \}$, which represents a set of accessible continuous states $s \in \sspace$
    the robot can be in, continuous bounded actions $a \in \aspace$ that a robotic agent may take, rewards $r$, and a
    transition probability $p(s_{t+1} \ | \ s_t, a_t)$, which govern the probability of transitioning from one state
    $s_t$ to the next $s_{t+1}$ given an action $a_t$ applied at time $t$ (we assume a deterministic transition model).
    The goal of the MDP formulation is to find a mapping from state to action that maximizes the total reward acquired
    from interacting in an environment for some fixed amount of time. This can also be written as the following
    objective
    \begin{equation}\label{eq:mdp_obj}
        \pi^\star = \argmax_{\pi} \mathbb{E}_{a \sim \pi(\cdot \ | \ s)} \left[ \sum_{t=0}^{T-1} r(s_t) \right]
    \end{equation}
    where the solution is an optimal policy $\pi^\star$. In most common experience-based reinforcement learning
    problems, a stochastic policy $a \sim \pi(\cdot \ | \ s)$ is learned such that it maximizes the reward $r$ at a
    state $s$. Model-based approaches solve the MDP problem by modeling the transition function $s_{t+1} = f(s_t, a_t)$
    and the reward function $r_t = r(s_t)$\footnote{We exclude the dependency on the action for clarity as one could
    always append the state vector with the action and obtain the dependency.}, and either use the model to construct a
    policy or directly generate actions through model-based planning~\cite{chua2018deep}. If the transition model and
    the reward function are known, the MDP formulation becomes an optimal control problem where one can use any set of
    existing methods~\cite{li2004iterative} to solve for the best set of actions (or policy) that
    maximizes the reward (often optimal control problems are specified using a cost instead of reward, however, the
    analysis remains the same).

    \paragraph{Hybrid Control Theory for Mode Scheduling:} In the problem of mode scheduling, the goal is to maximize a
    reward function through the synthesis of two (or more) control strategies (modes). This is achieved by optimizing
    when one should switch which policy is in control at each instant (which is developed from hybrid control theory).
    The policy switching often known as mode switching~\cite{axelsson2008gradient}. For example, a vertical take-off and
    landing vehicle switching from landing to flight mode is a common example used for an aircraft switching from flight
    mode to landing mode. Most problems are written in continuous time subject to continuous time dynamics of the form
    \begin{equation} \label{eq:cont_model}
        \dot{s}(t) = f(s(t), a(t)) = g(s(t)) + h(s(t))a(t)
    \end{equation}
    where $f(s, a) : \sspace \times \aspace \to \sspace $ is the (possibly nonlinear) transition function divided into
    the free dynamics $g(s) : \sspace \to \sspace$ and the control map $h(s) : \sspace \to \sspace \times \aspace$, and
    $a(t)$ is assumed to be generated from some default behavior. The objective function (similar to
    Eq.(\ref{eq:mdp_obj}) but often written in continuous time) is written as the deterministic function of the state
    and control:
    \begin{equation}\label{eq:cont_t_obj}
        \argmax_{\tau, \lambda} \mathcal{J} = \int_{t=0}^{t_H} r(s(t)) dt \ \ \ \text{subject to} \ \ \ \dot{s}(t) = f(s(t), a(t)), \ s(0) = s_0
    \end{equation}
    where
    \begin{equation}\label{eq:switching_action}
        a(t) =
            \begin{cases}
                \hat{a}(t),         & \text{if } t \in [ \tau, \tau + \lambda ] \\
                a_\text{def}(t)     & \text{otherwise}
            \end{cases},
    \end{equation}
    $t_H$ is the time horizon (in seconds), and $s(t)$ is generated from some initial condition $s(0)$ using the model
    (\ref{eq:cont_model}) and action sequence (\ref{eq:switching_action}). The goal in hybrid control theory is to find
    the optimal time $\tau$ and application duration $\lambda$ to switch from some default set of actions $a_\text{def}$ to
    another set of actions $\hat{a}$ that best improves the performance in (\ref{eq:cont_t_obj}) subject to the dynamics
    (\ref{eq:cont_model}).

    The following section derives the algorithmic combination of the MDP learning formulation with the
    hybrid control foundations into a joint hybrid learning approach.

    \begin{algorithm}[!h]
    \caption{Hybrid Learning (deterministic)}
    \begin{algorithmic}[1]
        \footnotesize
        \State Randomly initialize continuous differentiable models $f$, $r$ with parameters $\psi$ and policy $\pi$ with
        parameter $\theta$. Initialize memory buffer $\mathcal{D}$, prediction horizon parameter $t_H$, exploration noise
        $\varepsilon$.
        \While{task not done}
            \State reset environment and exploration noise $\varepsilon$
            \For{$i = 1, \ldots, T$}
                \State observe state $s(t_i)$
                \State $\triangleright$ simulation loop
                \For{$\tau_i \in \left[t_i, \ldots, t_i + t_H \right]$}
                    \State $\triangleright$ forward predict states using any integration method (Euler shown)
                    \State $s(\tau_{i+1}), r(\tau_i) = s(\tau_i) + f(s(\tau_i), \mu(s(\tau_i))) dt, \ r(s(\tau_i), \mu(s(\tau_i)))$
                \EndFor
                \State $\triangleright$ backwards integrate using $\dot{\rho}(t)$ defined in (\ref{eq:adjoint})
                \State $\rho(t_i + t_H) = \mathbf{0}$
                \For{$\tau_i \in \left[t_H + t_i, \ldots, t_i \right]$}
                    \State $\rho(\tau_{i-1}) = \rho(\tau_i) - \dot{\rho}(\tau_i) dt$
                \EndFor
                \State $a^\star(t_i) = \Sigma(s(t_i)) h(s(t_i))^\top \rho(t_i) + \mu(s(t_i)) + \varepsilon (t)$
                \State $\triangleright$ apply to robot
                \State append data $ \mathcal{D} \gets \{ s(t_i), a^\star (t_i), r_t, s(t_{i+1}) \} $
            \EndFor
            \State Update $f, r$ by sampling $N$ data points from $\mathcal{D}$ using any regression method
            \State Update $\pi$ using any experience-based method
        \EndWhile
    \end{algorithmic}
    \label{alg:hlt_det}
    \end{algorithm}

\vspace{-3mm}
\section{Hybrid Learning} \label{sec:hybrid_learning_theory}
\vspace{-2mm}

    The goal of this section is to introduce hybrid learning as a method for optimally utilizing model-based and
    experience-based (policy) learning. We first start with the deterministic variation of the algorithm that provides
    theoretical proofs, which describe the foundations of our method. The stochastic variation is then derived as a
    method for relaxing the assumptions made in the deterministic variation. The main theme in both the deterministic
    and stochastic derivations is that the learning problem is solved \textit{indirectly}. That is, we solve the (often
    harder) learning problem by instead solving sub-problems whose solutions imply that the harder problem is solved.

    \vspace{-2mm}
    \subsection{Deterministic} Consider the continuous time formulation of the objective and dynamics in
    (\ref{eq:cont_model}) and (\ref{eq:cont_t_obj}) with the MDP formation where $f$ and $r$ are learned using arbitrary
    regression methods (e.g., neural network least squares, Gaussian processes), and $\pi$ is learned through a
    experience-based approach (e.g., policy gradient~\cite{sutton2000policy}). In addition, let us assume
    that in the default action in (\ref{eq:switching_action}) is defined as the mean of the policy $\pi$ where we ignore
    uncertainty for the time being\footnote{We will add the uncertainty into the hybrid problem in the stochastic
    derivation of our approach for hybrid learning}. That is, $a_\text{def}(t) = \mu(s(t))$ is defined by assuming that the
    policy has the form $\pi(a \ | \ s) = \mathcal{N}(\mu(s), \Sigma(s))$, where $\mathcal{N}$ is a normal distribution
    and $\mu(s)$, $\Sigma(s)$ are the mean and variance of the policy as a function of state. As we do not have the form
    of $a^\star$, let us first calculate how sensitive (\ref{eq:cont_t_obj}) is at any $\tau$ to switching from $\mu(s)
    \to a^\star$ for an infinitely small $\lambda$\footnote{We avoid the problem of instability of the robotic system
    from switching control strategies as later we develop and use the best action for all $\tau \in [0, t_H]$ instead of
    searching for a particular time when to switch.}.

    \begin{lemma} \label{lemma:mode_insertion}
        Assume that $f$, $r$, and $\mu$ are differentiable and continuous in time. The sensitivity of
        (\ref{eq:cont_t_obj}) (also known as the mode insertion gradient~\cite{axelsson2008gradient}) with respect to the duration time $\lambda$
        from switching between $\mu(s)$ to $\hat{a}$ and any time $\tau \in [0, t_H]$ is defined as
        \begin{equation} \label{eq:mode_insertion}
            \frac{\partial}{\partial \lambda} \mathcal{J}(\tau) = \rho(\tau)^\top (f_2 - f_1)|_\tau
        \end{equation}
        where $f_1 = f(s(t), \mu(s(t)))$ and $f_2 = f(s(t), \hat{a}(t))$, and $\rho(t) \in \sspace$ is the adjoint
        variable which is the the solution to the the differential equation
        \begin{equation}\label{eq:adjoint}
            \dot{\rho}(t) = -\dldx - \left( \dfdx + \dmudx^\top \dfdu \right)^\top \rho(t)
        \end{equation}
        with terminal condition $\rho(t_H) = \mathbf{0}$.
    \end{lemma}
    \begin{proof}
        See Appendix~\ref{proof:mode_insert} for proof. \qed
    \end{proof}

    Lemma~\ref{lemma:mode_insertion} gives us the proof and definition of the mode insertion gradient
    (\ref{eq:mode_insertion}), which tells us the infinitesimal change in the objective function when switching from the
    default policy behavior to some other arbitrarily defined control $\hat{a}$ for a small time duration $\lambda$. We
    can directly use the mode insertion gradient to see how an arbitrary action changes the performance of the task from
    the policy that is being learned. However, in this work we use the mode insertion gradient as a method for obtaining
    the best action the robot can take given the learned predictive model of the dynamics and the task rewards. We can
    be more direct in our approach and ask the following question. \textbf{Given a suboptimal policy $\pi$, what is the
    best action the robot can take to maximize (\ref{eq:cont_t_obj}), at any time $t\in [0, t_H]$, subject to the
    uncertainty (or certainty) of the policy defined by $\Sigma(s)$?}

    We approach this new sub-problem by specifying the auxiliary optimization problem:
    \begin{equation}\label{eq:aux_opt}
        a^\star(t) = \argmax_{\hat{a}(t) \ \forall t \in [0, t_H]} \int_0^{t_H} \frac{\partial}{\partial \lambda} \mathcal{J}(t)
        + \log \pi \left( \hat{a} (t) \ | \ s(t) \right) dt
    \end{equation}
    where the idea is to maximize the mode insertion gradient (i.e., find the action that has the most impact in changing the objective)
    subject to the $\log \pi$ term that ensures the generated action $a^\star$ is penalized for deviating from the policy,
    when there is high confidence that was based on prior experience.

    \begin{theorem}\label{thm:optimal_action}

        Assuming that $f$, $r$, and $\pi$ are continuous and differentiable in $s, a$ and $t$, the best possible action that improves the
        performance of (\ref{eq:cont_t_obj}) and is a solution to (\ref{eq:aux_opt}) for any time $t \in [ 0, t_H ]$ is
        \begin{equation}\label{eq:optimal_action}
            a^\star(t) = \Sigma(s(t)) h(s(t))^\top \rho(t) + \mu(s(t))
        \end{equation}
        where $\rho(t)$ is defined in (\ref{eq:adjoint}) and $h(s): \mathbb{R}^n \to \mathbb{R}^{n \times m}$ is the
        affine mapping from actions to the dynamics.

    \end{theorem}

    \begin{proof}
        Inserting the definition of the mode insertion gradient (\ref{eq:mode_insertion}) and taking the derivative of
        (\ref{eq:aux_opt}) with respect to the point-wise $\hat{a}$ and setting it to zero gives
        \begin{equation*}
             \rho^\top h(s) \left( \hat{a} - \mu(s) \right) -  \Sigma(s)^{-1} \left( \hat{a} - \mu(s) \right) = 0
        \end{equation*}
        where we drop the dependency on time for clarity.
        Solving for $\hat{a}$ gives the best actions $a^\star$
        \begin{equation*}
            a^\star(t) = \Sigma(s(t)) h(s(t))^\top \rho(t) + \mu(s(t))
        \end{equation*}
        which is the action that maximizes the mode insertion gradient subject to the certainty of the policy $\pi$
        for all $t \in [0, t_H ]. $\qed
    \end{proof}

    The proof in Theorem~\ref{thm:optimal_action} provides the best action that a robotic system can take given a
    default experience-based policy. Each action generated uses the sensitivity of changing the objective based on the
    predictive model's behavior while relying on the experience-based policy to regulate when the model information will
    be useful. We convert the result in Theorem~\ref{thm:optimal_action} into our first (deterministic) algorithm (see
    Alg.~\ref{alg:hlt_det}).

    \begin{algorithm}[!h]
    \caption{Hybrid Learning (stochastic)}
    \begin{algorithmic}[1]
        \footnotesize
        \State Randomly initialize continuous differentiable models $f$, $r$ with parameters $\psi$ and policy $\pi$ with
        parameter $\theta$. Initialize memory buffer $\mathcal{D}$, prediction horizon parameter $H$.
        \While{task not done}
            \State reset environment
            \For{$t = 1, \ldots, T-1$}
                \State observe state $s_t$
                \State $\triangleright$ simulation loop
                \For{$k \in \{ 0, \ldots, K-1 \}$}
                    \For{$\tau \in \{ 0, \ldots, H-1 \}$}
                        \State $v_\tau^k \sim \pi( \cdot \ | \ s_\tau^k)$
                        \State $\triangleright$ forward predict state and reward
                        \State $s_{\tau+1}^k, r_\tau^k = f(s_\tau^k, v_\tau^k), r(s_\tau^k, v_\tau^k)$
                        \State $j^k_\tau = r_\tau^k$
                    \EndFor
                \EndFor
                \State $\triangleright$ update actions
                \For{$\tau \in \{ 0, \ldots, T-1 \}$}
                    \State $\mathcal{J}(v_\tau^k) \gets \sum_{t=\tau}^{T-1} j_t^k$
                    \State $\delta a_\tau^k \gets v_\tau^k - a_\tau$
                    \State $a_\tau \gets a_\tau + \sum_{k=0}^{K-1} \omega(v_\tau^k) \delta a_\tau^k$
                \EndFor
                \State apply $a_0$ to robot and append data $ \mathcal{D} \gets \{ s_t, a_0, r_t, s_{t+1} \} $
            \EndFor
            \State Update $f, r$ by sampling $N$ data points from $\mathcal{D}$ using any regression method
            \State Update $\pi$ using any experience-based method
        \EndWhile
    \end{algorithmic}
    \label{alg:hlt_stoch}
    \end{algorithm}

    The benefit of the proposed approach is that we are able to make (numerically based) statements about the
    generated action and the contribution of the learned predictive models towards improving the task. Furthermore,
    we can even make the claim that (\ref{eq:optimal_action}) provides the best possible action given the current
    belief of the dynamics $f$ and the task reward $r$.

    \begin{corollary} \label{cor:Hamiltonian}

        Assuming that $\frac{\partial}{\partial a} \mathcal{H} \ne 0$ where $\mathcal{H} = r(s) + \log \pi(a \ | \ s) +
        \rho^\top f(s, a)$ is the control Hamiltonian for (\ref{eq:cont_t_obj}), then $\frac{\partial}{\partial \lambda}
        \mathcal{J} = \Vert h(s)^\top \rho \Vert_{\Sigma(s)} > 0$ and is zero when the policy satisfies
        the control Hamiltonian condition $\frac{\partial}{\partial a} \mathcal{H} = 0$.

    \end{corollary}
    \begin{proof}
        Inserting (\ref{eq:optimal_action}) into (\ref{eq:mode_insertion}) yields
        \begin{align}
            \frac{\partial \mathcal{J}}{\partial \lambda} &= \rho^\top \left(
                g(s) + h(s)\left(\Sigma(s) h(s)^\top \rho + \mu(s)\right) - g(s) - h(s)\mu(s)
            \right) \nonumber \\
            & = \rho^\top h(s) \Sigma(s) h(s)^\top \rho = \Vert h(s)^\top \rho \Vert_{\Sigma(s)} > 0.
        \end{align}
        From Pontryagin's Maximum principle, a solution that is a local optima of the objective function satisfies the
        following
        \begin{equation}
            \frac{\partial}{\partial a} \mathcal{H} = - \Sigma(s)^{-1} \left( a - \mu(s) \right) + h(s)^\top \rho = 0
        \end{equation}
        when $a = \Sigma(s) h(s)^\top \rho + \mu(s)$ or $\rho = 0$. Therefore, if the policy $\pi$ is a solution, then it
        must be that the adjoint $\rho=0$ and $\pi$ a solution to the optimal control problem (\ref{eq:cont_t_obj}). \qed
    \end{proof}
    Corollary~\ref{cor:Hamiltonian} tells us that the action defined in (\ref{eq:optimal_action}) generates the best action
    that will improve the performance of the robot given valid learned models. In addition, Corollary~\ref{cor:Hamiltonian}
    also states that if the policy is already a solution, then our approach for hybrid learning does not impede on the
    solution and returns the policy's action.

    Taking note of each proof, we can see that there is the strict requirement of continuity and differentiability of
    the learned models and policy. As this is not always possible, and often learned models have noisy derivatives, our
    goal is to try to reformulate (\ref{eq:cont_t_obj}) into an equivalent problem that can be solved without the need
    for the assumptions. One way is to formulate the problem in discrete time (as an expectation), which we will do in
    the following section.

    \vspace{-3mm}
    \subsection{Stochastic} We relax the continuity, differentiability, and continuous-time restrictions specified
    (\ref{eq:cont_t_obj}) by first restating the objective as an expectation:
    \begin{equation}
        \max \mathbb{E}_{v \sim \pi( \cdot \ | \ s)} \left[
            \mathcal{J}(v)
            \right]
    \end{equation}
    where $\mathcal{J}(v) = \sum_{t=0}^{T-1} r(s_t)$ subject to $s_{t+1} = f(s_t, v_t)$, and $v = [v_0, \ldots v_{H-1}]$
    is a sequence of $H$ randomly generated actions from the policy $\pi$. Rather than trying to find the best time
    $\tau$ and discrete duration $\lambda$, we approach the problem from an hybrid information theoretic view and
    instead find the best augmented actions to $\pi$ that improve the objective. This is accomplished by defining two
    distributions $\mathbb{P}$ and $\mathbb{Q}$ which are the uncontrolled system response distribution\footnote{We
    refer to uncontrolled as the unaugmented control response of the robotic agent subject to a stochastic policy $\pi$}
    and the open loop control distribution (augmented action distribution) described as probability density functions
    $p(v) = \prod_{t=0}^{T-1} \pi \left( v_t \ | \ s_t \right)$ and $q(v) = \prod_{t=0}^{T-1} \frac{1}{\sqrt{(2 \pi)^m |
    \Sigma(s_t) |}} \exp \left( -\frac{1}{2} (v_t - a_t)^\top \Sigma(s_t)^{-1} (v_t - a_t) \right)$ respectively. Here,
    $\pi( a \ | \ s) = \mathcal{N} (\mu(s), \Sigma(s) )$ and $q(v)$ use the same variance $\Sigma(s)$ as the policy. The
    uncontrolled distribution $\mathbb{P}$ represents the default predicted behavior of the robotic system under the
    learned policy $\pi$. Furthermore, the augmented open-loop control distribution $\mathbb{Q}$ is a way for us to
    define a probability of an augmented action, but more importantly, a free variable for which to optimize over given
    the learned models. Following the work in~\cite{williams2017information}, we use Jensen's inequality and importance
    sampling on the free-energy~\cite{theodorou2012relative} definition of the control system using the open loop
    augmented distribution $\mathbb{Q}$:
    \begin{equation} \label{eq:free_energy}
        \mathcal{F}(v) = -\lambda \log \left(
                    \mathbb{E}_{\mathbb{P}} \left[ \exp \left(\frac{1}{\lambda} \mathcal{J}(v) \right) \right]
                    \right)
                     \le - \lambda \mathbb{E}_{\mathbb{Q}} \left[
                            \log \left(  \frac{p(v)}{q(v)} \exp \left( \frac{1}{\lambda} \mathcal{J}(v) \right) \right)
                        \right]
    \end{equation}
    where $\lambda \in \mathbb{R}^+$ here is what is known as the temperature parameter (and not the time duration as used prior).
    Note that in (\ref{eq:free_energy}) if $\frac{p(v)}{q(v)} \propto 1/\exp\left( \frac{1}{\lambda}
    \mathcal{J}(v)\right)$ then the inequality becomes a constant. Further reducing the free-energy gives the following:
    \begin{equation}
        \mathcal{F}(v)  \le - \lambda \mathbb{E}_{\mathbb{Q}} \left[
                \log \left(  \frac{p(v)}{q(v)} \exp \left( \frac{1}{\lambda} \mathcal{J}(v) \right) \right)
            \right] \nonumber
         \le - \mathbb{E}_\mathbb{Q} \left[ \mathcal{J}(v)
            - \lambda \log \left( \frac{p(v)}{q(v)} \right) \right]
    \end{equation}
    which is the optimal control problem we desire to solve plus a bounding term which binds the augmented actions to
    the policy. In other words, the free-energy formulation can be used as an indirect approach to solve for the hybrid
    optimal control problem by making the bound a constant. Specifically, we can indirectly make $\frac{p(v)}{q(v)}
    \propto 1/\exp\left( \frac{1}{\lambda} \mathcal{J}(v)\right)$ which would make the free-energy bound reduce to a
    constant. Using this knowledge, we can define an optimal distribution $\mathbb{Q}^\star$ through its density
    function
    \begin{equation}
        q^\star(v) = \frac{1}{\eta} \exp \left( \frac{1}{\lambda} \mathcal{J} (v) \right) p(v), \ \
        \eta = \int_{\Omega} \exp \left( \frac{1}{\lambda} \mathcal{J} (v)\right) p(v)dv
    \end{equation}
    where $\Omega$ is the sample space.\footnote{The motivation is to use the optimal density function to gauge how well
    the policy $\pi$ performs.} Letting the ratio be defined as $\frac{p(v)}{q^\star(v)}$ gives us the proportionality
    that we seek to make the free-energy a constant. However, because we can not directly sample from
    $\mathbb{Q}^\star$, and we want to generate a separate set of actions $a_t$ defined in $q(v)$ that augments the
    policy given the learned models, our goal is to push $q(v)$ towards $q^\star(v)$. As done
    in~\cite{williams2017information, williams2016aggressive} this corresponds to the following optimization:
    \begin{equation}\label{eq:stoch_opt_prob}
        a^\star = \argmin_{a} D_\text{KL} \left(\mathbb{Q}^\star \ | \ \mathbb{Q} \right)
    \end{equation}
    which minimizes the Kullback-Leibler divergence of the optimal distribution $\mathbb{Q}^\star$ and the open-loop distribution
    $\mathbb{Q}$. In other words, we want to construct a separate distribution that augments the policy distribution
    $p(v)$ (based on the optimal density function) such that the objective is improved.

    \begin{theorem}
        The recursive, sample-based, solution to (\ref{eq:stoch_opt_prob}) is
        \begin{align}
            a^\star_t = a_t + \sum_k \omega(v^k_t) \delta a^k_t
             \ \ \text{where} \ \  \omega(v)
            = \frac{\exp \left( \frac{1}{\lambda} \mathcal{J} (v) \right) p(v)}{\sum_n \exp \left(  \frac{1}{\lambda} \mathcal{J} (v)  \right) p(v)}
        \end{align}
        where $k$ denotes the sample index and $v_t = a_t + \delta a_t$.
    \end{theorem}
    \begin{proof}
        See Appendix~\ref{proof:hybrid_stochastic} for proof \qed
    \end{proof}

    The idea behind the stochastic variation is to generate samples from the stochastic policy and evaluate its utility
    based on the current belief of the dynamics and the reward function. Since samples directly depend on the likelihood
    of the policy, any actions that steer too far from the policy will be penalized depending on the confidence of the
    policy. The inverse being when the policy has low confidence (high variance) the sample span increases and more
    model-based information is utilizes. Note that we do not have to worry about continuity and differentiability
    conditions on the learned models and can utilize arbitrarily complex models for use of this algorithm. We
    outline the stochastic algorithm for hybrid learning Alg.~\ref{alg:hlt_stoch}.

\vspace{-3mm}
\section{Results} \label{sec:results}
\vspace{-2mm}

    In this section, we present two implementations of our approach for hybrid learning. The first uses experience-based
    methods through robot interaction with the environment. The goal is to show that our method can improve the overall
    performance and sample-efficiency by utilizing learned predictive models and experience-based policies generated
    from off-policy reinforcement learning. In addition, we aim to show that through our approach, both the model-based
    and experience-based methods are improved through the hybrid synthesis. The second implementation illustrates our
    algorithm with imitation learning where expert demonstrations are used to generate the experience-based policy
    (through behavior cloning) and have the learned predictive models adapt to the uncertainty in the policy. All
    implementation details are provided in Appendix~\ref{app:imp}.

    \begin{figure}[h!]
        \centering
        \includegraphics[width=0.95\linewidth]{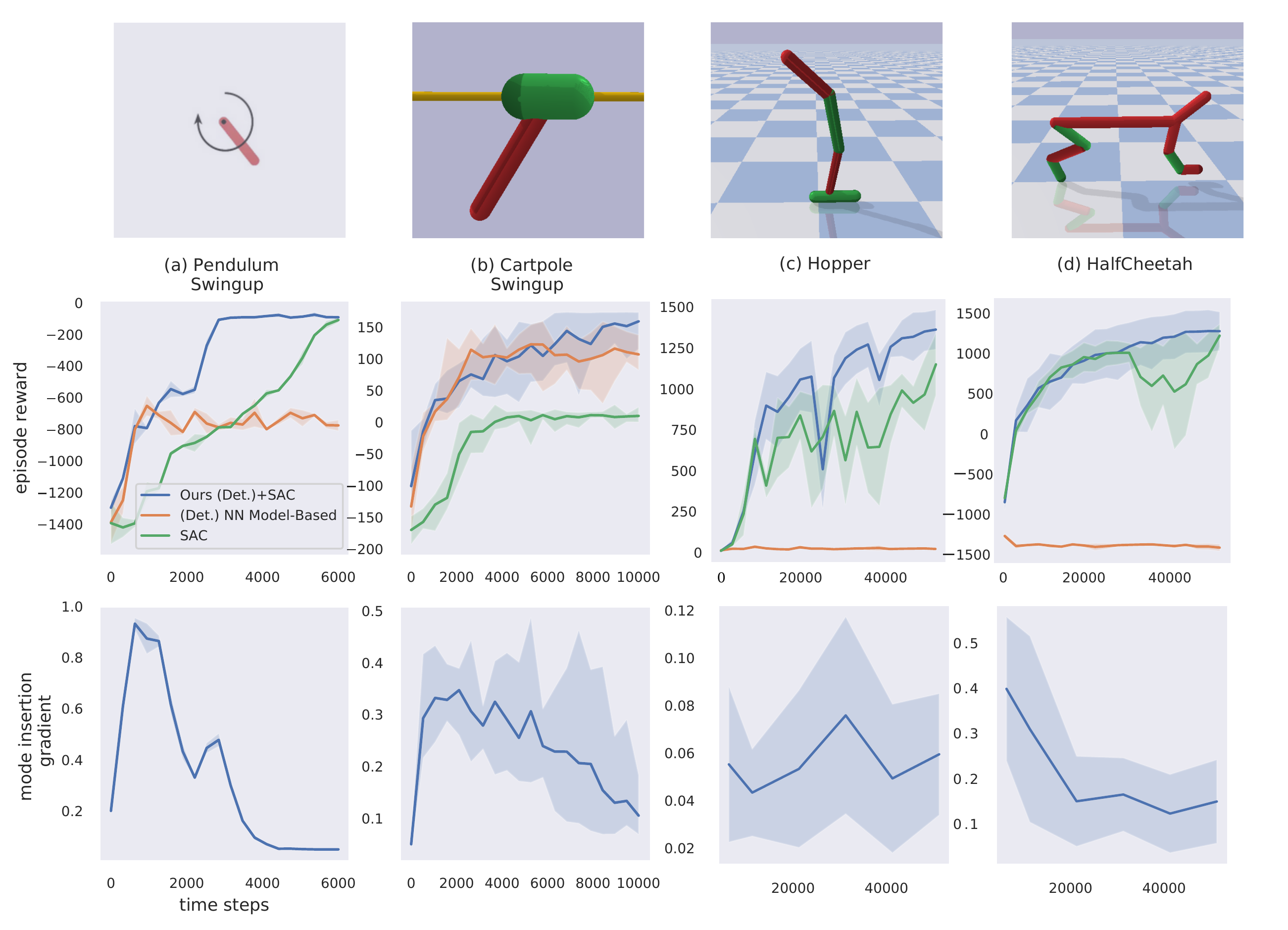}
        \vspace{-4mm}
        \caption{
            \footnotesize Performance curves of our proposed deterministic hybrid learning algorithm on multiple
            environments (averaged over 5 random seeds). All methods use the same structured learning models.
            Our method is shown to improve the model-based benchmark results (due to the use of experience-based methods)
            while maintaining significant improvements on the number of interactions necessary with the environment to
            obtain those results. The mode insertion gradient is also shown for each example which illustrates the
            model-policy agreement over time and the improvement over time.
        }
        \label{fig:deter_bm_results}
    \end{figure}

    \begin{figure}[h!]
        \vspace{-2mm}
        \centering
        \includegraphics[width=0.95\linewidth]{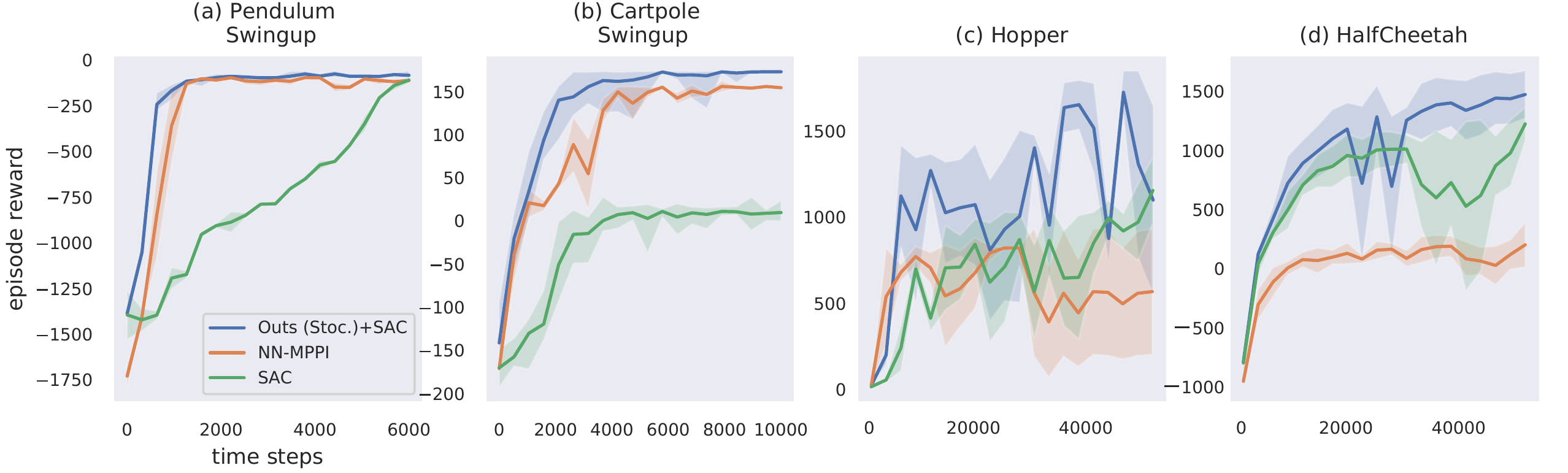}
        \vspace{-4mm}
        \caption{
            \footnotesize Performance curves of our proposed stochastic hybrid learning algorithm on multiple
            environments (averaged over 5 random seeds). As shown before, our approach improves both the sample-efficiency but
            also the highest expected reward. In addition, the stochastic variation of the hybrid learning algorithm
            generates smoother learning curves as a result of not requiring derivatives of learned models.
        }
        \label{fig:stoch_bm_results}
        \vspace{-6mm}
    \end{figure}

    \vspace{-2mm}
    \paragraph{Learning from Experience:} We evaluate our approach in the deterministic and stochastic settings using
    experience-based learning methods and compare against the standard in model-based and experience-based learning. In
    addition, we illustrate the ability to evaluate our method's performance of the learned models through the mode
    insertion gradient. Experimental results validate hybrid learning for real robot tasks. For each example, we use
    Soft Actor Critic (SAC)~\cite{haarnoja2018soft} as our experience-based method and a neural-network based
    implementation of model-predictive path integral for reinforcement learning~\cite{williams2017information} as a
    benchmark standard method. The parameters for SAC are held as default across all experiments to remove any impact of
    hyperparameter tuning.

    Hybrid learning is tested in four simulated environments: pendulum swingup, cartpole swingup, the hopper
    environment, and the half-cheetah environment (Fig.~\ref{fig:deter_bm_results}) using the Pybullet
    simulator~\cite{coumans2016pybullet}. In addition, we compare against state-of-the-art approaches for model-based
    and experience-based learning. We first illustrate the results using the deterministic variation of hybrid learning
    in Fig.~\ref{fig:deter_bm_results} (compared against SAC and a deterministic model-predictive
    controller~\cite{ansari2016sequential}). Our approach uses the confidence bounds generated by the stochastic policy
    to infer when best to rely on the policy or predictive models. As a result, hybrid learning allows for performance
    comparable to experience-based methods with the sample-efficiency of model-based learning approaches. Furthermore,
    the hybrid control approach allows us to generate a measure for calculating the agreement between the policy and the
    learned models (bottom plots in Fig.~\ref{fig:deter_bm_results}), as well as when (and how much) the models were
    assisting the policy. The commonality between each example is the eventual reduction in the assistance of the
    learned model and policy. This allows us to better monitor the learning process and  dictate how well the policy is
    performing compared to understanding the underlying task.

    \begin{figure}
        \vspace{-5mm}
        \centering
        \includegraphics[width=\linewidth]{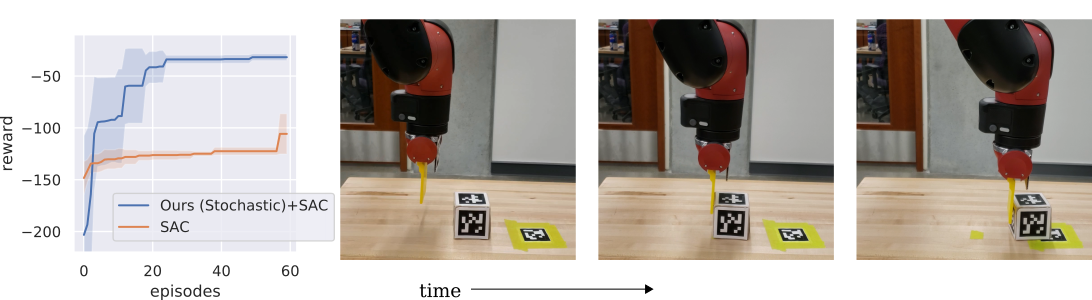}
        \vspace{-5mm}
        \caption{
            \footnotesize Hybrid learning with experience-based policy results on the Sawyer robot (averaged over 5 trials). The task is to push
            a block to a designated target through environment interactions (see time-series results above). Our method is
            able achieve the task within 3 minutes (each episode takes 10 seconds) through effectively using both
            predictive models and experience-based methods. The same amount of interaction with SAC was unable to
            successfully push the block to the target.
        }
        \label{fig:sawyer_hlt}
        \vspace{-5mm}
    \end{figure}

    We next evaluate the stochastic variation of hybrid learning, where we compare against a
    stochastic neural-network model-based controller~\cite{williams2017information} and SAC. As shown in
    Fig.~\ref{fig:stoch_bm_results}, the stochastic variation still maintains the improved performance and
    sample-efficiency across all examples while also having smoother learning curves. This is a direct result of the
    derivation of the algorithm where continuity and differentiability of the learned models are not necessary. In
    addition, exploration is naturally encoded into the algorithm through the policy, which results in more stable
    learning when there is uncertainty in the task. In contrast, the deterministic approach required added exploration noise
    to induce exploring other regions of state-space. A similar result is found when comparing the model-based performance
    of the deterministic and stochastic approaches, where the deterministic variation suffers from modeling discontinuous
    dynamics.

    We can analyze the individual learned model and policy in Fig.~\ref{fig:stoch_bm_results} obtained from hybrid
    learning. Specifically, we look at the cartpole swingup task for the stochastic variation of hybrid learning in
    Fig.~\ref{fig:component_analysis} and compare against benchmark model-based learning
    (NN-MPPI~\cite{williams2017information}) and experience-based learning (SAC~\cite{haarnoja2018soft}) approaches.
    Hybrid learning is shown to improve the learning capabilities of both the learned predictive model and the policy
    through the hybrid control approach. In other words, the policy is ``filtered'' through the learned model and
    augmented, allowing the robotic system to be guided by both the prediction and experience. Thus, the predictive
    model and the policy are benefited, ultimately performing better as a standalone approach using hybrid learning.

    \begin{SCfigure}
        \centering
        \includegraphics[width=0.45\linewidth]{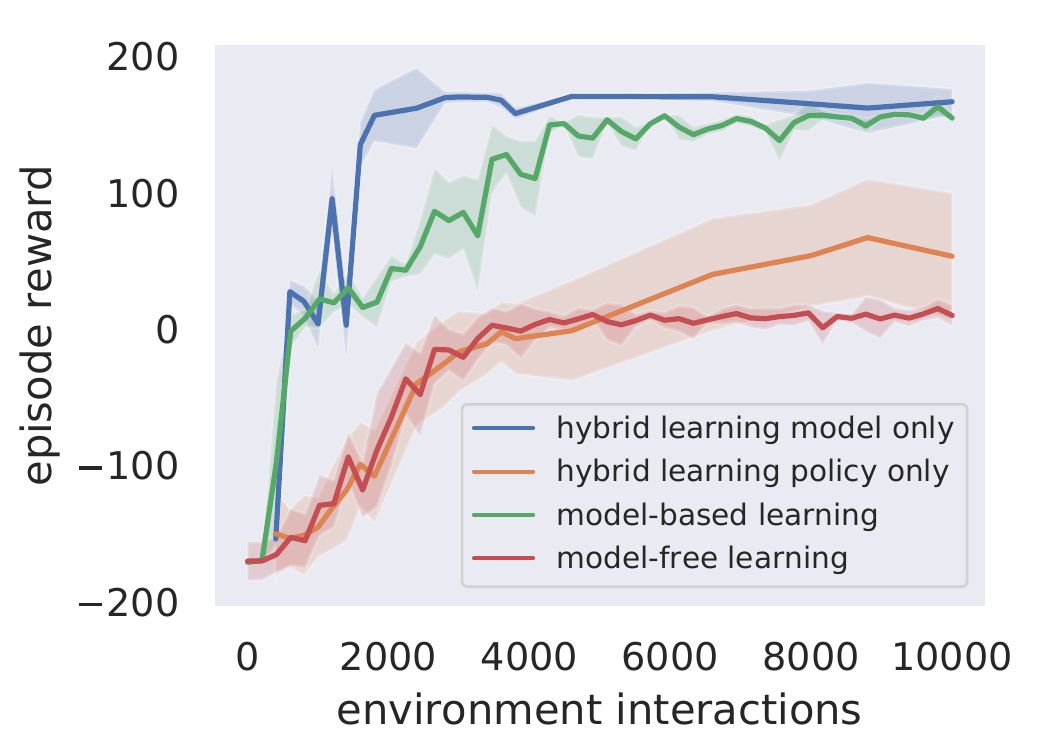}
        \vspace{3mm}
        \caption{
            \footnotesize Performance curves for the individual learned model and policy on the cartpole swingup
            environment during hybrid learning (averaged over 10 trials). Our method is shown to improve the capabilities
            of the model-based and experience-based components through mutual guidance defined by hybrid control theory.
            Reference model-based learning (NN-MPPI) and experience-based learning (SAC) approaches are shown for
            comparison.
        }
        \label{fig:component_analysis}
    \end{SCfigure}

    Next, we apply hybrid learning on real robot experiments to illustrate the sample-efficiency and performance our
    approach can obtain (see Fig.~\ref{fig:sawyer_hlt} for task illustration). \footnote{The same default parameters for
    SAC are used tor this experiment.} We use a Sawyer robot whose goal is to push a block on a table to a specified
    marker. The position of the marker and the block are known to the robot. The robot is rewarded for pushing the block
    to the marker. What makes this task difficult is the contact between the arm and the block that the robot needs to
    discover in order to complete the pushing task. Shown in Fig.~\ref{fig:sawyer_hlt} our hybrid learning approach is
    able to learn the task within 20 episodes (total time is 3 minutes, 10 seconds for each episode). Since our method
    naturally relies on the predictive models when the policy is uncertain, the robot is able to plan through the
    contact to achieve the task whereas SAC takes significantly longer to discover the pushing dynamics.
    As a result, we are able to achieve the task with minimal environment interaction.

    \begin{figure}
        \centering
        \includegraphics[width=\linewidth]{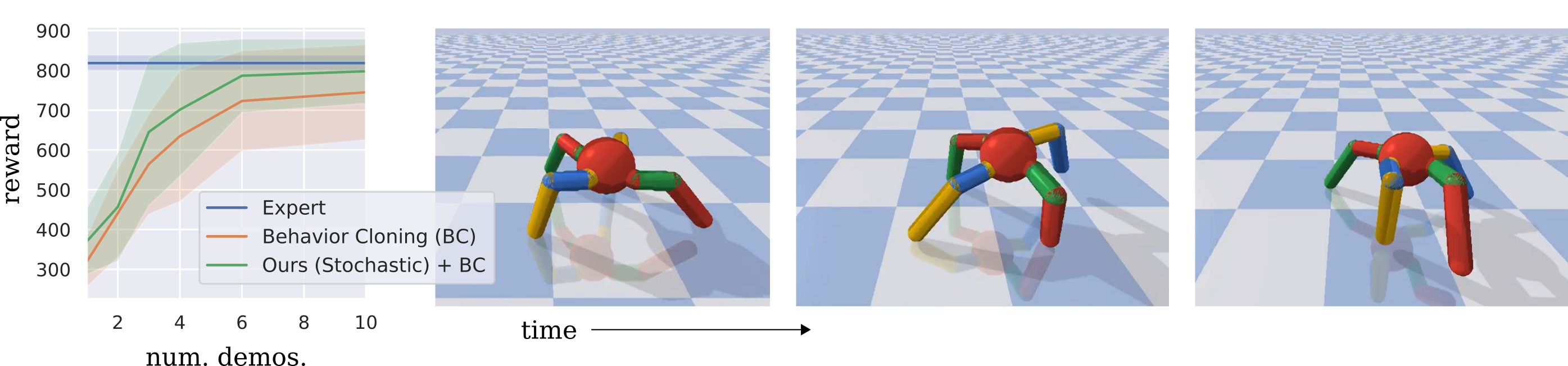}
        \vspace{-4mm}
        \caption{
            \footnotesize Results for hybrid stochastic control with behavior cloned policies (averaged over 10 trials)
            using the Ant Pybullet environment (shown in a time-lapsed running sequence). Expert demonstrations (actions
            executed by an expert policy on the ant robot) are used as experience to boot-strap a learned stochastic
            policy (behavior cloning) in addition to predictive models which encode the dynamics and the underlying task
            of the ant. Our method is able to adapt the expert experience to the predictive models, improving the
            performance of behavior cloning and performing as well as the expert.
        } \vspace{-2mm}
        \label{fig:ant_bc}
    \end{figure}
    \vspace{-8mm}

    \paragraph{Learning from Examples:} We extend our method to use expert demonstrations as experience (also known as
    imitation learning~\cite{argall2009survey, ross2010efficient}). Imitation learning focuses on using expert
    demonstrations to either mimic a task or use as initialization for learning complex data-intensive tasks. We use
    imitation learning, specifically behavior cloning, as an initialization for how a robot should accomplish a task.
    Hybrid learning as described in Section~\ref{sec:hybrid_learning_theory} is then used as a method to embed
    model-based information to compensate for the uncertainty in the learned policy, improving the overall performance
    through planning. The specific algorithmic implementation of hybrid imitation learning is provided in
    Appendix~\ref{app:imp}.

    \begin{figure}[h!]
        \centering
        \includegraphics[width=\linewidth]{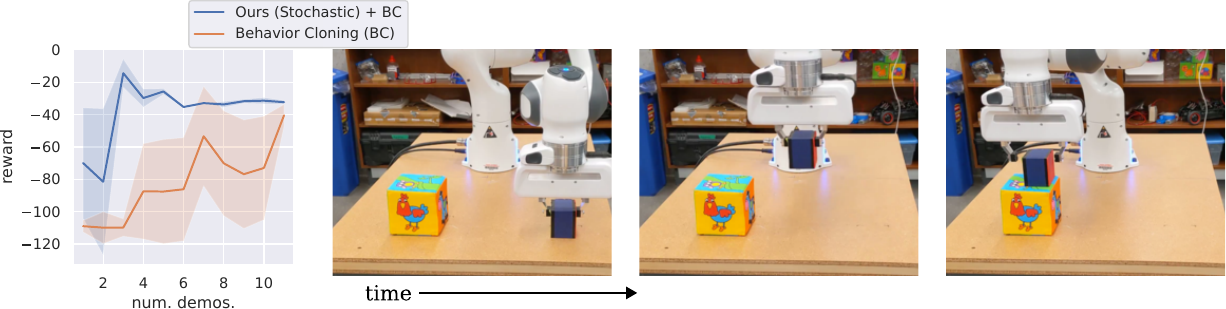}
        \vspace{-5mm}
        \caption{
            \footnotesize
            Hybrid learning with behavior cloning results on the Franka panda robot (averaged over 5 trials).
            The task is to stack a block on top of another using expert demonstrations. Our method
            is able to learn the block stacking task within three expert demonstrations and
            provides solutions that are more repeatable than with behavior cloning.
        }
        \label{fig:franka_bc}
        \vspace{-6mm}
    \end{figure}

    We test hybrid imitation on the Pybullet Ant environment. The goal is for the four legged ant to run as far as it
    can to the right (from the viewer's perspective) within the allotted time. At each iteration, we provide the agent
    with an expert demonstration generated from a PPO~\cite{schulman2017proximal} solution. Each demonstration is used
    to construct a predictive model as well as a policy (through behavior cloning). The stochastic hybrid learning
    approach is used to plan and test the robot's performance in the environment. Environment experience is then used to
    update the predictive models while the expert demonstrations are solely used to update the policy. In
    Fig.~\ref{fig:ant_bc}, we compare hybrid learning against behavior cloning. Our method is able to achieve the task
    at the level of the expert within 6 (200 step) demonstrations, where the behavior cloned policy is unable to achieve
    the expert performance. Interestingly, the ant environment is less susceptible to the covariant shift problem (where
    the state distribution generated by the expert policy does not match the distribution of states generated by the
    imitated policy \cite{ross2010efficient}), which is common in behavior cloning. This suggests that the ant
    experiences a significantly large distribution of states during the expert demonstration. However, the resulting
    performance for the behavior cloning is worse than that of the expert. Our approach is able to achieve similar
    performance as behavior cloning with roughly 2 fewer demonstrations and performs just as well as the expert
    demonstrations.

    We test our approach on a robot experiment with the Franka Panda robot (which is more likely to have the covariant
    shift problem). The goal for the robot is to learn how to stack a block on top of another block using demonstrations
    (see Fig.~\ref{fig:franka_bc}). As with the ant simulated example in Fig.~\ref{fig:ant_bc}, a demonstration is
    provided at each attempt at the task and is used to update the learned models. Experience obtained in the
    environment is solely used to update the predictive models. We use a total of ten precollected demonstrations of the
    block stacking example (given one at a time to the behavior cloning algorithm before testing). At each testing time,
    the robot arm is initialized at the same spot over the initial block. Since the demonstrations vary around the arm's
    initial position, any state drift is a result of the generated imitated actions and will result in the covariant
    shift problem leading to poor performance. As shown in Fig.~\ref{fig:franka_bc}, our approach is capable of learning
    the task in as little as two demonstrations where behavior cloning suffers from poor performance. Since our approach
    synthesizes actions when the policy is uncertain, the robot is able to interpolate between regions where the expert
    demonstration was lacking, enabling the robot to achieve the task.

\vspace{-3mm}
\section{Conclusion}\label{sec:conclusion}
\vspace{-1mm}

    We present hybrid learning as a method for formally combining model-based learning with experience-based policy
    learning based on hybrid control theory. Our approach derives the best action a robotic agent can take given the
    learned models. The proposed method is then shown to improve both the sample-efficiency of the learning process as
    well as the overall performance of both the model and policy combined and individually. Last, we tested our approach
    in various simulated and real-world environments using a variety of learning conditions and show that our method
    improves both the sample-efficiency and the resulting performance of learning motor skills.
\vspace{-2mm}
{\footnotesize
\bibliography{references}
}

\subfile{appendix/appendix}

\end{document}

%% file: custom_commands.tex
\newcommand{\dldx}{\frac{\partial r}{\partial s}}

\newcommand{\dfdx}{\frac{\partial f}{\partial s}}
\newcommand{\dfdu}{\frac{\partial f}{\partial a}}

\newcommand{\dmudx}{\frac{\partial \mu}{\partial s}}

\newcommand{\MDP}{\mathcal{M}}
\newcommand{\sspace}{\mathcal{S}}
\newcommand{\aspace}{\mathcal{A}}

%% file: appendix/appendix.tex
\mainmatter

\title{Appendix: Hybrid Control for\\ Learning Motor Skills}
\titlerunning{Appendix: Hybrid Control for Learning Motor Skills}  

\author{Ian Abraham\inst{1} \and Alexander Broad\inst{2} \and Allison Pinosky\inst{1} \and \\ Brenna Argall\inst{1,2} \and Todd D. Murphey\inst{1}}
\authorrunning{Abraham et al.} 
%
\tocauthor{Ian Abraham, Alexander Broad,
Allison Pinosky, Brenna Argall, and Todd D. Murphey}

\institute{
    Department of Mechanical Engineering \\
    \and
    Department of Electrical Engineering \\
    and Computer Science \\
    Northwestern University, Evanston, IL 60208, USA \\
    \email{i-abr@u.northwestern.edu, alexsbroad@gmail.com,
    apinosky@u.northwestern.edu, t-murphey@northwestern.edu, brenna.argall@northwestern.edu}
}

\maketitle

\section{Proofs} \label{app:proofs}

    \begin{lemma}
        Assume that $f$, $r$, and $\mu$ are differentiable and continuous in time. The sensitivity of
        (\ref{eq:cont_t_obj}) (also known as the mode insertion gradient) with respect to the duration time $\lambda$
        from switching between $\mu(s)$ to $a^\star$ and any time $\tau \in [0, t_H]$ is defined as
        \begin{equation}
           \frac{\partial}{\partial \lambda} \mathcal{J}(\tau) = \rho(\tau)^\top (f_2 - f_1)|_\tau
        \end{equation}
        where $f_1 = f(s(t), \mu(s(t)))$ and $f_2 = f(s(t), a^\star(t))$, and $\rho(t) \in \sspace$ is the adjoint
        variable which is the the solution to the the differential equation
        \begin{equation}
            \dot{\rho}(t) = -\dldx - \left( \dfdx + \dmudx^\top \dfdu \right)^\top \rho(t)
        \end{equation}
        with terminal condition $\rho(t_H) = \mathbf{0}$.
    \end{lemma}
    \begin{proof}\label{proof:mode_insert}
        First define the trajectory
        \begin{equation}\label{eq:hybrid_trajectory}
            s(t_H) = s(0) + \int_0^\tau f(s(t), \mu(s(t))) dt + \int_\tau^{\tau + \lambda} f(s(t), a^\star(t)) dt \
            + \int_{\tau + \lambda}^{t_H} f(s(t), \mu(s(t)))dt
        \end{equation}
        generated from $a(t) =
                    \begin{cases}
                        a^\star(t),         & \text{if } t \in [ \tau, \tau + \lambda ] \\
                        a_\text{def}(t)     & \text{otherwise}
                    \end{cases}$ where $a_\text{def}(t) = \mu(s(t))$.
        Next, let us take the derivative of (\ref{eq:cont_t_obj}) with respect to the time duration $\lambda$ so that
        we have the following expression:
        \begin{equation}\label{eq:mi_1}
            \frac{\partial}{\partial \lambda} \mathcal{J} = \int_{\tau + \lambda}^{t_H}
                \dldx^\top\frac{\partial s}{ \partial \lambda} dt.
        \end{equation}
        Using (\ref{eq:hybrid_trajectory}), we can define $\frac{\partial s}{\partial \lambda}$ as
        \begin{equation}\label{eq:recursive_dsdlam}
            \frac{\partial s(t)}{\partial \lambda} = f_2 - f_1 + \int_{\tau + \lambda}^t
            \left( \dfdx + \dmudx^\top\dfdu \right)^\top \frac{\partial s (\sigma)}{\partial \lambda} d\sigma
        \end{equation}
        where $\sigma$ is a place holder for time under the integrand, and $f_1 = f(s(t), \mu(s(t)))$ and $f_2 = f(s(t),
        a^\star(t))$ are remaining boundary terms from applying Leibniz's rule.

        Noting that (\ref{eq:recursive_dsdlam}) is a linear convolution (due to the repeating $\frac{\partial s}{\partial \lambda}$ terms)
        with initial condition $\frac{\partial s}{\partial \lambda}(\tau) = f_2 - f_1$, we can rewrite (\ref{eq:recursive_dsdlam}) using
        a state-transition matrix
        \begin{equation}
            \Phi(t, \tau)=\exp \left( \left( \dfdx + \dmudx^\top\dfdu \right)^\top (t - \tau) \right)
        \end{equation}
        with initial condition $f_2-f_1$ as
        \begin{equation}\label{eq:convolution}
            \frac{\partial s(t)}{\partial \lambda} = \Phi(t, \tau)(f_2 - f_1)
        \end{equation}
        Using (\ref{eq:convolution}) in (\ref{eq:mi_1}) and pulling out the term $f_2-f_1$ from under the integrand, we can rewrite
        (\ref{eq:mi_1}) as the following:
        \begin{equation}
            \frac{\partial}{\partial \lambda} \mathcal{J}(\tau) = \lim_{\lambda \to 0} \int_{\tau + \lambda}^{t_H}
                \frac{\partial r}{\partial s}^\top \Phi(t, \tau) dt \left(f_2 - f_1\right).
        \end{equation}
        Taking the limit as $\lambda \to 0$ gives the instantaneous sensitivity from switching from $\mu \to a^\star$ at any time $\tau$.
        Let us define this term as the adjoint variable
        \begin{equation}
            \rho(\tau)^\top = \int_\tau^{t_H} \frac{\partial r}{\partial s}^\top \Phi(t, \tau) dt
        \end{equation}
        which give us the mode insertion gradient
        \begin{equation}
            \frac{\partial}{\partial \lambda} \mathcal{J}(\tau) = \rho(\tau)^\top \left(f_2 - f_1 \right)
        \end{equation}
        where the adjoint can be rewritten as the following differential equation
        \begin{equation}
            \dot{\rho}(t) = -\dldx - \left( \dfdx + \dmudx^\top \dfdu\right) \rho(t)
        \end{equation}
        with terminal condition $\rho(t_H) = \mathbf{0}$. \qed
    \end{proof}

    \begin{theorem}
        The recursive, sample-based, solution to
        \begin{equation}
            a^\star = \argmin_a D_\text{KL}\left( \mathbb{Q}^\star \Vert \mathbb{Q}\right)
        \end{equation}
        found in Eq.~\ref{eq:stoch_opt_prob} is
        \begin{align}
            a^\star_t = a_t + \sum_k \omega(v_t) \delta a_t
             \ \ \text{where} \ \  \omega(v)
            = \frac{\exp \left( \frac{1}{\lambda} \mathcal{J} (v) \right) p(v)}{\sum_n \exp \left(  \frac{1}{\lambda} \mathcal{J} (v)  \right) p(v)}
        \end{align}
        where $k$ denotes the sample index and $v_t = a_t + \delta a_t$.
    \end{theorem}
    \begin{proof}\label{proof:hybrid_stochastic}
        Expanding the objective in (\ref{eq:stoch_opt_prob}), we can show that
        \begin{align}\label{eq:astar_obj}
            a^\star & = \argmin_a \mathbb{E}_{\mathbb{Q}^\star} \left[
                        \log \left( \frac{q^\star(v)}{q(v)} \right)\right] \nonumber\\
                    & =  \argmin_a \int_\Omega q^\star(v)
                                \log \left( \frac{q^\star(v)}{p(v)} \frac{p(v)}{q(v)} \right) dv \nonumber\\
                    & = \argmin_a  \int_\Omega q^\star(v) \log \left(\frac{q^\star(v)}{p(v)} \right)
                            - \int_\Omega q^\star(v) \log \left( \frac{q(v)}{p(v)} \right) dv \nonumber\\
                    & = \argmax_a \int_\Omega q^\star(v) \log\left( \frac{q(v)}{p(v)} \right) dv .
        \end{align}
        Defining the policy $\pi(v_t \ | \ s_t) = \mathcal{N} \left( \mu(s_t), \Sigma(s_t) \right)$ as normally distributed,
        we can show that
        \begin{align*}
            \frac{q(v)}{p(v)} & \propto \exp\left(\sum_t -\frac{1}{2} (v_t - a_t)^\top \Sigma^{-1} (v_t - a_t)
            + \frac{1}{2}(v_t - \mu(s_t))^\top \Sigma^{-1} (v_t - \mu(s_t)) \right) \\
            & = \exp\left(\sum_t -\frac{1}{2} (v_t - a_t)^\top \Sigma^{-1} (v_t - a_t)
            + \frac{1}{2}(v_t - \mu(s_t))^\top \Sigma^{-1} (v_t - \mu(s_t)) \right) \\
            & = \exp\left(\sum_t -\frac{1}{2} a_t^\top \Sigma^{-1} a_t + a_t^\top \Sigma^{-1} v_t
                    + \mu(s_t)^\top \Sigma^{-1} (\mu(s_t) - 2 v_t) \right)
        \end{align*}
        where $\Sigma = \Sigma(s)$ is used as short-hand notation.
        Plugging this expression into Eq.~(\ref{eq:astar_obj}) gives
        \begin{equation*}
            a^\star = \argmax_a \sum_t -\frac{1}{2} a_t^\top \Sigma^{-1} a_t
                    + a_t^\top \int_\Omega q^\star(v) \Sigma^{-1}v_t dv
                    + \mu(s_t)^\top\int_\omega q^\star(v)(\mu(s_t) - 2v_t) dv.
        \end{equation*}
        which we can solve for $a_t$ at each time by setting the derivative with respect
        to $a_t$ to zero to give the optimal solution
        \begin{equation}\label{eq:opt_sol1}
            a_t^\star = \int_\Omega q^\star(v) v_t dv.
        \end{equation}
        Note that the expression $q^\star(v) \propto \exp\left( \frac{1}{\lambda} \mathcal{J} (v)\right) p(v)$ which
        allows us to rewrite (\ref{eq:opt_sol1}) in the following way:
        \begin{align}
            a_t^\star &= \int_\Omega q^\star(v) v_t dv
                = \int_\Omega \frac{1}{\eta} \exp \left(\frac{1}{\lambda}\mathcal{J}(v) \right)p(v) v_t dv \\
                & = \mathbb{E}_\mathbb{P} \left[ \frac{1}{\eta} \exp \left(\frac{1}{\lambda}\mathcal{J}(v) \right) v_t \right].
        \end{align}
        Using the change of variable $v_t = a_t + \delta a_t$, we get the recursive, sample-based solution
        \begin{equation}
            a^\star_t = a_t + \sum_k \omega(v^k_t) \delta a_t^k
        \end{equation}
        where
        \begin{equation}
            \omega(v) = \frac{\exp \left( \frac{1}{\lambda} \mathcal{J} (v) \right) p(v)}{\sum_n \exp \left(  \frac{1}{\lambda} \mathcal{J} (v)  \right) p(v)}.
        \end{equation}
        \qed
    \end{proof}

\section{Implementation Details} \label{app:imp}

    Here, we present implementation details for each of the examples provided in the main paper as well as additional
    algorithmic details presented and mentioned throughout the paper. Any parameter not explicitly mentioned as
    deterministic or stochastic variations of hybrid learning are equivalent unless otherwise specified. All
    simulated examples have reward functions specified as the default rewards found in the Pybullet
    environments~\cite{coumans2016pybullet} unless otherwise specified. Table~\ref{tab:param} provides
    a lists of all hyperparameters used for each environment tested.

    \paragraph{Models:} For each simulated example using the experience-based method, we use the same model
    representation of the dynamics as $s_{t+1} = s_t + f(s_t ,a_t)$ where $f(s_t, a_t) = \mathbf{W}_2\sin(\mathbf{W}_1
    [s_t,a_t] + \mathbf{b}_1) + \mathbf{b}_2$, and $\mathbf{W}_1 \in \mathbb{R}^{200 \times (n+m)}$, $\mathbf{W}_2 \in
    \mathbb{R}^{n \times 200}$, $\mathbf{b}_1 \in \mathbb{R}^{200}$, $\mathbf{b}_2 \in \mathbb{R}^n$ are learned
    parameters. For locomotion tasks we use the rectifying linear unit (RELU) nonlinearity. The reward function is
    modeled as a two layer network with $200$ hidden nodes and rectifying linear unit activation function. Both the
    reward function and dynamics model are optimized using Adam~\cite{kingma2014adam} with a learning rate of $0.003$.
    The model is regularized using the negative log-loss of a normal distribution where the variance,
    $\Sigma_\text{model} \in \mathbb{R}^{n \times n}$, is a hyperparameter that is simultaneously learned based on
    model-free learning. The predicted reward utility is improved by the error between the predicted target and target
    reward equal to $\mathcal{L} = \Vert r_t + 0.95 \ r(s_{t+1}, a_{t+1}) - r(s_t, a_t) \Vert ^2$ (similar to the
    temporal-difference learning~\cite{boyan1999least, precup2001off}). This loss encourages learning the value of the
    state and action that was taken for environments that have rewards that do not strictly depend on the current state
    (i.e., the reward functions used in Pybullet locomotion examples). A batch size of $128$ samples are taken from the
    data buffer $\mathcal{D}$ for training.

    \setlength\tabcolsep{6pt} 
    \begin{table}[]
    \centering
    \begin{tabular}{ccccccc}
    \toprule
    Environment      & $H$ & $T$ & $K$ & $\lambda$ & policy dim & nonlinearity \\ \midrule
    Pendulum Swingup & 5           & 200              & 20  & 0.1   &  $128$   & $\sin(x)$ \\ \hline
    Cartpole Swingup & 5           & 200              & 20  & 0.1   &  $128$   & $\sin(x)$ \\ \hline
    Hopper           & 5           & 1000             & 20  & 0.1   &  $128$   & $\text{relu}(x)$ \\ \hline
    Half-Cheetah     & 10          & 1000             & 20  & 0.2   &  $128$   & $\text{relu}(x)$ \\ \hline
    Sawyer           & 10          & 100              & 10  & 0.01   & $128 \times 128$     & $\text{relu}(x)$  \\ \hline
    Ant              & 20          & 400              & 40  & 1.0   & $128 \times 64$    & $\text{relu}(x)$ \\ \hline
    Franka Panda     & 40          & 200              & 40  & 1.0   & $32 \times 24$     & $\text{relu}(x)$  \\ \hline
    \bottomrule
    \end{tabular}
    \vspace{5mm}
    \caption{
        Parameters for all examples used in the paper (only when applicable). Each example using the deterministic
        variation of hybrid learning (Alg.~\ref{alg:hlt_det}) uses added action exploration of the form $\epsilon=0.999^t$
        where $t$ is the total number of environment interactions.
    }
    \label{tab:param}
    \end{table}

    \paragraph{Policy:} For the policy, we parameterize a normal distribution with a mean function defined as a single
    layer network with $\sin(x)$ nonlinearity with $128$ nodes (similar to the dynamics model used). The diagonal of the
    variance is specified using a single layer with $128$ nodes and rectifying linear unit activation function. Soft
    actor critic (SAC) is used to optimize the policy for the pendulum, cartpole, hopper, and half-cheetah environments
    respectively. All examples use the same hyperparameters for SAC specified by the shared parameters in~\cite{haarnoja2018soft}
    including the structure of the value and soft Q functions, and excluding the batch size and policy (which we
    match the $128$ samples used with model learning and to utilize the simpler policy representation).

    The ant and panda robot with behavior cloning use the policy structure defined in Table~\ref{tab:param}, which
    is structured in a similar stochastic parameterization as mentioned in the paragraph above. The negative log loss
    of the normal distribution is used for behavior cloning expert demonstrations with a learning rate of $0.01$ for
    each method.

    \paragraph{Robot Experiments:} In all robot experiments, a camera is used to identify the location of objects in the
    environment using landmark tags and color image processing. For the Sawyer robot example, the state is defined as the pose
    of the end-effector of the arm to the block as well as the pose of the block to the target. The action space is the
    target end-effector velocity. The reward is defined as
    \begin{equation*}
        r(s, a) = - 5 \Vert p_\text{b2t} \Vert - 10 \Vert p_\text{ee2b}\Vert - 0.01 \Vert a \Vert^2
    \end{equation*}
    where $p_\text{b2t},  p_\text{ee2b}$ denote the poses of the block to the target and the end-effector to the target
    location respectively.

    For the Franka robot, the state is defined as the end-effector position, the block position, and the gripper state
    (open or closed) as well as the measured wrench at the end-effector. The action space is defined as the commanded end-effector
    velocity. The reward function is defined as
    \begin{equation*}
        r(s,a) = r_\text{stage}(s) - 1.0e^{-6} \left( \Vert F_\text{ee} \Vert + \Vert a \Vert \right)
    \end{equation*}
    where
    \begin{equation*}
        r_\text{stage}(s) =
        \begin{cases}
            -1.25 \Vert p_\text{ee} - p_\text{stack} \Vert , & \text{if grasped block} \\
            - \Vert p_\text{ee} - p_\text{block} \Vert ,         & \text{if not grasped block}
        \end{cases}
    \end{equation*}
    denotes the stage at which the Franka is in at the block stacking task. Here, $p_\text{ee}, p_\text{block}, p_\text{stack}$, and $F_\text{ee}$
    denote the end-effector pose, the block pose, the target stacking position, and the measured wrench at the end-effector respectively.

    \begin{algorithm}[!h]
    \caption{Hybrid Learning (stochastic) with Behavior Cloning}
    \begin{algorithmic}[1]
        \State Randomly initialize continuous differentiable models $f$, $r$ with parameters $\psi$ and policy $\pi$ with
        parameter $\theta$. Initialize memory buffer $\mathcal{D}$ and expert data buffer $\mathcal{D}_\text{exp}$,
        prediction horizon parameter $H$.
        \While{task not done}
            \State $\triangleright$ get expert demonstrations
            \For{$t = 0, \ldots, T-1$}
                \State observe state $s_t$, expert action $a_t$
                \State observe $s_{t+1}, r_t$ from environment
                \State $\mathcal{D}_\text{exp} \gets \{ s_t, a_t, r_t, s_{t+1}\}$
                \State $\mathcal{D}  \gets \{ s_t, a_t, r_t, s_{t+1}\}$
            \EndFor
            \State $\triangleright$ update models using data
            \State update $\psi$ using $\mathcal{D}$ any regression method
            \State update $\theta$ using $\mathcal{D}_\text{exp}$ with behavior cloning
            \State $\triangleright$ test in environment
            \For{$t = 0, \ldots, T-1$}
                \State observe state $s_t$
                \State get action $a_t$ Alg.~\ref{alg:hlt_det} or~\ref{alg:hlt_stoch}
                \State observe $s_{t+1}, r_t$ from environment
                \State $\mathcal{D}  \gets \{ s_t, a_t, r_t, s_{t+1}\}$
            \EndFor
            \State if task not done, continue
        \EndWhile
    \end{algorithmic}
    \label{alg:hybrid_imitation}
    \end{algorithm}
